\documentclass[12pt]{article}

\usepackage{amsmath, amssymb, amsthm}
\usepackage{optidef}
\usepackage[backend=biber, style=authoryear-icomp, maxcitenames=2, maxbibnames=99, ibidtracker=false]{biblatex}
\usepackage{hyperref}
\usepackage{bm}
\usepackage{bbm}
\usepackage{graphicx}
\usepackage{booktabs}
\usepackage[a4paper, width=150mm, top=30mm, bottom=30mm]{geometry}
\usepackage{setspace}
\usepackage{abstract}
\usepackage{xcolor}

\DeclareMathOperator*{\argmin}{arg\,min}
\DeclareMathOperator{\var}{Var}

\newtheorem{theorem}{Theorem}[section]
\newtheorem{proposition}[theorem]{Proposition}
\newtheorem{corollary}[theorem]{Corollary}

\theoremstyle{definition}

\newtheorem{remark}[theorem]{Remark}

\begin{document}

\title{\large\textbf{A Mathematical Optimization Approach for Expert-Informed Bayesian Best Subset Selection}}

\author{%
    \begin{tabular}[t]{c@{\hspace{3em}}c}
        \textbf{Nolan Alexander} & \textbf{Henning Mortveit} \\
        Systems and Information & Systems and Information \\
        Engineering & Engineering \\
        University of Virginia & University of Virginia \\
        \texttt{nka5we@virginia.edu} & \texttt{hsm2v@virginia.edu}
    \end{tabular}
}

\date{June 2026}
\maketitle

\begin{abstract}
\noindent
A central challenge in statistical modeling is identifying the subset of features that belong in the true regression model. The classical best subset selection problem, recently made tractable via mixed-integer optimization (MIO), finds the globally optimal sparse solution. It does not, however, make use of any information beyond the observed data. In many applied settings, domain experts can meaningfully rank or score the relevance of candidate predictors, yet no existing framework integrates such probabilistic expert assessments directly into the best-subsets objective. This paper presents \emph{Expert-Implied Bayesian Best Subsets} (EBBS), a method that incorporates domain-expert probability estimates of feature relevance into the MIO best-subsets problem through a maximum \textit{a posteriori}(MAP) framework. Expert views from multiple respondents are aggregated into a single prior probability per feature using the Poisson binomial distribution for marginal probability estimates, the pairwise win rate for pairwise comparisons, or the normalized mean rank for ordinal rankings. This probability enters the objective function as a log-odds penalty term that smoothly encourages or discourages the selection of each feature consistent with the expert consensus. This paper provides analytic derivations of the MAP formulation and characterizes its theoretical properties. The proposed model reduces to Best Subsets when experts all have no views. Empirical results on synthetic and real datasets are forthcoming.
\end{abstract}

\vspace{0.5em}
\noindent
\textbf{Keywords:} best subset selection, mixed-integer optimization, feature selection, Bayesian inference, maximum a posteriori, domain knowledge integration, sparse linear regression.

\section{Introduction}
\label{sec:intro}

We consider the linear regression model $\bm{y} = \bm{X}\boldsymbol{\beta} + \boldsymbol{\varepsilon}$, where $\bm{y} \in \mathbb{R}^n$ is the response vector, $\bm{X} \in \mathbb{R}^{n \times p}$ is the design matrix, $\boldsymbol{\beta} \in \mathbb{R}^p$ are regression coefficients, and~$\boldsymbol{\varepsilon} \in \mathbb{R}^n$ are errors. In many classical and modern applications, it is desirable to obtain a parsimonious fit by identifying the $k$-feature model that best explains~$\bm{y}$, especially in the high-dimensional regime $p \gg n$ where the true $\boldsymbol{\beta}$ is assumed to be sparse. On the other hand, practitioners in fields such as medicine, epidemiology, and financial economics frequently possess informed prior beliefs about which predictors are likely to matter. Purely data-driven selection cannot exploit such beliefs.

\paragraph{Brief context and background.}
There are three main classes of feature selection methods: automated, manual, and hybrid methods. Among automated approaches, there are four sub-classes: heuristics, regularization, information-theoretic methods, and exact solutions.

Heuristic stepwise methods \parencite{efroymsonMultipleRegressionAnalysis1960} iteratively add or remove predictors based on test statistics. They tend to underperform global optimization approaches because of their greedy nature. Regularization methods, including Ridge \parencite{hoerlRidgeRegressionApplications1970}, Lasso \parencite{tibshiraniRegressionShrinkageSelection1996}, Elastic Net \parencite{zouRegularizationVariableSelection2005}, and extensions such as the Relaxed Lasso \parencite{meinshausenRelaxedLasso2007}, Adaptive Lasso \parencite{zouAdaptiveLassoIts2006}, and SCAD \parencite{fanVariableSelectionNonconcave2001}, modify the objective with continuous penalties that induce sparsity through shrinkage of coefficients. Information-theoretic approaches such as mRMR \parencite{pengFeatureSelectionBased2005} and QPFS \parencite{rodriguez-lujanQuadraticProgrammingFeature} select features by maximizing relevance, often measured as the correlation to the predictand, while minimizing redundancy, often measured as the cross-correlation between predictors. Both regularization and information-theoretic approaches optimize a different objective than the standard regression sum of squared errors. As such, they offer no guarantee of finding the globally optimal sparse subset for standard regression. The exact approach, best subset selection, finds the globally optimal $k$-feature model by directly minimizing the residual sum of squares subject to a hard cardinality constraint. \textcite{bertsimasBestSubsetSelection2016} showed that Best Subset Selection can be framed as a mixed-integer program (MIP) and solved to provable optimality at problem sizes with $n$ in the thousands and $p$ in the hundreds. \textcite{hazimehFastBestSubset2020} subsequently developed coordinate-descent and local combinatorial algorithms to further improve computational speed.

Bayesian approaches to variable selection place priors on binary inclusion indicators and use Markov Chain Monte Carlo (MCMC) to explore the posterior over models. \textcite{georgeApproachesBayesianVariable1997} developed hierarchical mixture priors with independent Bernoulli inclusion indicators and Gibbs sampling. \textcite{kuoVariableSelectionRegression1998} embed indicators directly into the regression equation, yielding a spike-and-slab prior on the effective coefficients. \textcite{kowalBayesianSubsetSelection2022a} take a decision-analytic approach, deriving optimal linear coefficients for any subset from a continuous shrinkage model and collecting a family of near-optimal subsets via predictive evaluation. All three rely on MCMC for posterior computation. None incorporate domain-expert knowledge into the inclusion prior, and none produce a globally optimal sparse solution via direct optimization.

Several recent methods integrate domain knowledge with automated selection, though all do so by modifying penalized regression objectives. \textcite{zengIncorporatingPriorKnowledge2021} extend the Lasso with empirical-Bayes penalty parameters for each feature. \textcite{wuDomainKnowledgeenhancedVariable2022} develop an algorithm that can incorporate and reject domain knowledge in feature selection. \textcite{destouchesFeatureSelectionPrior2023} use iterative input rescaling in support vector regression. The method requires priors on $\boldsymbol{\beta}$ directly, which can be difficult for practitioners to specify. \textcite{zhangLLMLassoRobustFramework2025} leverage large language models to generate feature-level penalty weights for the Lasso via retrieval-augmented generation, with an internal cross-validation step that bounds performance relative to standard Lasso. All of these methods inherit the limitations of the Lasso family: they do not find the globally optimal sparse subset for standard regression. None can be interpreted as a MAP estimate under a clearly specified prior on the binary selection indicators.

\paragraph{Our approach.}
We propose \emph{Expert-Implied Bayesian Best Subsets} (EBBS), a framework that extends the MIP formulation of \textcite{bertsimasBestSubsetSelection2016} to incorporate probabilistic domain-expert views through a MAP framework. Each expert $i$ provides probability estimates $e_{i,a} \in [0,1]$ of whether feature $a$ belongs in the true model, where $0.5$ encodes no view, $1$ encodes certainty of inclusion, and $0$ encodes certainty of exclusion. These individual estimates are aggregated into a single per-feature probability $\bar{e}_a$. Depending on the elicitation format, this uses the Poisson Binomial distribution for marginal probability estimates, the pairwise win rate for pairwise comparisons, or the normalized mean rank for ordinal rankings. We then place a Bernoulli prior with parameter $\bar{e}_a$ on the binary selection indicator $z_a$. Deriving the MAP estimator yields a MIP that is identical in structure to the best-subsets MIP with an additional log-odds term $-\bm{z}^\top \log(\bm{\bar{e}} / (\mathbf{1} - \bm{\bar{e}}))$ in the objective. This term is computed once from the elicited views and enters the solver as a fixed coefficient vector. All warm-starting and MIO advances of \textcite{bertsimasBestSubsetSelection2016} apply without modification. When $\bar{e}_a = 0.5$ for all features, EBBS reduces to standard best subsets.

\paragraph{Contributions.}
We summarize our contributions below:
\begin{enumerate}
    \item We derive a MAP formulation of best subset selection that incorporates probabilistic domain-expert views on feature relevance, yielding a MIP that extends \textcite{bertsimasBestSubsetSelection2016} with a single additional linear term.
    \item We propose aggregating expert views via the Poisson Binomial distribution for marginal probability estimates, the pairwise win rate for pairwise comparisons, and the normalized mean rank for ordinal rankings, each providing a direct and principled mapping from individual expert inputs to a single consensus prior per feature.
    \item We characterize the theoretical properties of the resulting objective, establishing monotone behavior in the expert prior, boundary behavior, and the reduction to standard best subsets under uninformative priors.
    \item We design a synthetic and real-data empirical validation strategy, with results forthcoming.
\end{enumerate}

The structure of the paper is as follows. Section~\ref{sec:method} presents the MIP formulation of best subsets, develops the EBBS model, and derives the MAP optimization problem. Section~\ref{sec:properties} establishes methodological properties of the formulation. Section~\ref{sec:empirical} outlines the empirical validation strategy, and Section~\ref{sec:conclusion} includes our concluding remarks.

\section{Expert-Implied Bayesian Best Subsets}
\label{sec:method}

\subsection{Best Subset Selection via MIO}
\label{sec:best_subsets}

We start with the MIP formulation of \textcite{bertsimasBestSubsetSelection2016}, which forms the foundation of our proposed method. Let $\bm{z} \in \{0,1\}^p$ be a binary vector indicating feature inclusion, $k$ be the maximum number of active features, and $\mathcal{M}_U$ be a constant exceeding the magnitude of any optimal coefficient. The best-subsets problem is
\begin{equation}
\label{eq:best_subsets_opt}
\begin{aligned}
    \min_{\boldsymbol{\beta},\, \bm{z}}
    \quad & \frac{1}{2} \| \bm{y} - \bm{X} \boldsymbol{\beta} \|_2^2 \\
    \textrm{s.t.}
    \quad & -\mathcal{M}_U z_i \leq \beta_i \leq \mathcal{M}_U z_i, \qquad i = 1, \ldots, p, \\
          & z_i \in \{0, 1\}, \qquad i = 1, \ldots, p, \\
          & \sum_{i=1}^p z_i \leq k.
\end{aligned}
\end{equation}
The constraint $-\mathcal{M}_U z_i \leq \beta_i \leq \mathcal{M}_U z_i$ forces $\beta_i = 0$ when $z_i = 0$ while leaving~$\beta_i$ unconstrained when $z_i = 1$. Our method preserves the constraint structure of Eq.~\eqref{eq:best_subsets_opt} and modifies only the objective.

\subsection{Expert View Specification}
\label{sec:expert_views}

There are multiple possible formats for eliciting domain-expert views. All three formulations described below ultimately produce the aggregated inclusion probability vector $\bar{\boldsymbol{e}}$ required by the MAP objective in Eq.~\eqref{eq:novel_opt}. They differ only in how expert knowledge is collected.

\subsubsection{Marginal Probability Elicitation}
\label{sec:marginal_elicitation}

This formulation requires each expert to provide a probability estimate of whether each feature belongs to the true model. An estimate of $0$ encodes near-certainty that the feature is absent from the true model. An estimate of $1$ encodes near-certainty that it is present. An estimate of $0.5$ corresponds to the expert holding no view on that feature. These probabilities are collected into the expert view
matrix $\bm{e} \in [0,1]^{N_{\text{experts}} \times p}$, where each column corresponds to a feature and each row corresponds to an expert.

\subsubsection{Pairwise Comparison Elicitation}
\label{sec:pairwise_elicitation}

An alternative formulation that may be more natural for domain experts requires each expert to make binary judgments over pairs of features, indicating which of two features is more important for predicting the response. Let $c_{i,ab} \in \{0,1\}$ denote the judgment of expert $i$ that feature $a$ is more important than feature $b$, with~$c_{i,ab} + c_{i,ba} = 1$ for all pairs $(a,b)$ with $a \neq b$. These judgments are collected across all~$N_{\text{e}}$ experts and up to $\binom{p}{2}$ feature pairs into a judgment array $\bm{C}$. Experts may leave pairs blank if they have no view on the relative importance of those two features or believe the two have equal importance. These blank pairs are omitted from the aggregation and contribute no information. This approach does not require transitivity of judgments.

\subsubsection{Ordinal Ranking Elicitation}
\label{sec:ordinal_elicitation}

A third formulation requires each expert to provide a complete or partial ranking of the $p$ features from most to least important. Let $\pi_i$ denote the ranking provided by expert $i$, where $\pi_i(r)$ is the feature assigned rank $r$. These rankings are collected across all $N_{\text{e}}$ experts. This formulation imposes a lower elicitation burden than the pairwise formulation when the feature set is large. A ranking of $p$ features requires only $p - 1$ decisions rather than up to $\binom{p}{2}$ pairwise comparisons. Experts who can only rank a subset of features can be accommodated by treating their response as a partial ranking over the features they have assessed. Unlike the pairwise comparisons, this approach requires a partial linear ordering of features.

\subsection{Expert View Aggregation}
\label{sec:aggregation}

Each elicitation format described in Section~\ref{sec:expert_views} requires a different aggregation method to produce the per-feature inclusion probability vector $\bar{\boldsymbol{e}} \in (0,1)^p$. In all cases, $\bar{e}_a$ encodes the degree of expert consensus that feature $a$ belongs in the true model, and is used identically as the parameter of the Bernoulli prior on $z_a$ in Section~\ref{sec:map}.

\subsubsection{Aggregation of Marginal Probabilities via the Poisson Binomial Distribution}
\label{sec:agg_marginal}

When experts provide marginal probability estimates, the estimates from different experts are treated as independent Bernoulli trials with different success probabilities. The distribution of the total number of successes follows the Poisson Binomial distribution \parencite{wangNumberSuccessesIndependent1993}. Let $p_i$ denote the probability of success of the $i$th trial, and let $F_\ell$ denote the set of all subsets of $\ell$ integers that can be selected from $\{1, 2, \ldots, N_{\text{experts}}\}$. The cumulative distribution function of the Poisson Binomial distribution is
\[
    \mathbb{P}(X \leq k) = \sum_{\ell=0}^{k} \sum_{A \in F_\ell} \left( \prod_{i \in A} p_i \prod_{j \in A^c} (1 - p_j) \right).
\]
The individual probabilities in $\bm{e}$ are aggregated into $\bar{\boldsymbol{e}}$, where $\bar{e}_a$ is defined as the probability that a majority of experts favor the inclusion of feature $a$, with ties resolved by uniform tie-breaking. Let $S_a$ denote the realized number of experts that vote for feature $a$. The aggregate expert probability for feature $a$ is
\begin{equation}
\label{eq:e_bar}
    \bar{e}_a = \Pr\!\left( S_a > \tfrac{N_{\text{experts}}}{2} \right) + \tfrac{1}{2}\, \Pr\!\left( S_a = \tfrac{N_{\text{experts}}}{2} \right),
\end{equation}
where each term is obtained from the Poisson Binomial probability mass function
\[
    \Pr(S_a = \ell) = \sum_{B \in F_{\ell}} \prod_{i \in B} e_{i,a} \prod_{j \in B^c} (1 - e_{j,a}).
\]
When $N_{\text{e}}$ is odd, $N_{\text{e}}/2$ is not an integer so $\Pr(S_a = N_{\text{e}}/2) = 0$ and Eq.~\eqref{eq:e_bar} reduces to the strict-majority probability $\Pr(S_a > \lfloor N_{\text{experts}}/2 \rfloor)$. When $N_{\text{e}}$ is even, the tie event $S_a = N_{\text{e}}/2$ carries positive mass and is split evenly between inclusion and exclusion. The coefficient $1/2$ is the deterministic version of a fair tie-breaking rule. Without it, a strict-majority aggregator applied to uninformed reports $e_{i,a} = 0.5$ would yield $\bar{e}_a < 0.5$ for even $N_{\text{e}}$, introducing a spurious penalty against inclusion.

\subsubsection{Aggregation of Pairwise Comparisons via the Win Rate}
\label{sec:agg_pairwise}

When experts provide pairwise judgments, each comparison $c_{i,ab} = 1$ constitutes a vote that feature $a$ is more important than feature $b$. Let $n_{ab} = \sum_{i=1}^{N_{\text{experts}}} c_{i,ab}$ denote the total number of experts who judged feature $a$ more important than feature $b$, and let $\mathcal{C}_a = \{b : n_{ab} + n_{ba} > 0\}$ denote the set of features against which feature $a$ was compared by at least one expert. The win rate of feature $a$ is defined as the proportion of comparisons involving feature $a$ that feature $a$ won:
\begin{equation}
\label{eq:e_bar_pw}
    \bar{e}_a = \frac{\displaystyle\sum_{b \in \mathcal{C}_a} n_{ab}}{\displaystyle\sum_{b \in \mathcal{C}_a} (n_{ab} + n_{ba})}.
\end{equation}
Since $\bar{e}_a \in [0,1]$ by construction, it serves directly as the inclusion probability for feature $a$. If feature $a$ wins all of its comparisons, then $\bar{e}_a = 1$; if feature $a$ loses all comparisons, then $\bar{e}_a = 0$; and if feature $a$ has equal wins and losses, then $\bar{e}_a = 0.5$, which contributes no net incentive to the objective. The win rate handles non-transitive judgments across experts by averaging over all observed comparisons without requiring consistency. If feature $a$ is not involved in any comparison, $\bar{e}_a$ is set to $0.5$ by convention.

\subsubsection{Aggregation of Ordinal Rankings via the Normalized Mean Rank}
\label{sec:agg_ordinal}

When experts provide ordinal rankings, each ranking $\pi_i$ assigns a position from $1$ (most important) to $p$ (least important) to each feature. Let $\pi_i(a)$ denote the rank assigned to feature $a$ by expert $i$, and let $\mathcal{R}_a = \{i : \text{expert } i \text{ ranked feature } a\}$ denote the set of experts who ranked feature $a$. The mean rank of feature $a$ across all experts who ranked it is
\[
    \bar{\pi}_a = \frac{1}{|\mathcal{R}_a|} \sum_{i \in \mathcal{R}_a} \pi_i(a).
\]
This is then normalized to $[0,1]$ by mapping rank $1$ to~$\bar{e}_a = 1$ and rank $p$ to $\bar{e}_a = 0$, which gives us
\begin{equation}
\label{eq:e_bar_rank}
    \bar{e}_a = 1 - \frac{\bar{\pi}_a - 1}{p - 1}.
\end{equation}
A feature ranked first by all experts receives $\bar{e}_a = 1$, a feature ranked last by all experts receives $\bar{e}_a = 0$, and a feature with an average rank of $(p+1)/2$ receives $\bar{e}_a = 0.5$, contributing no net incentive to the objective. Features not ranked by any expert have $\bar{e}_a$ set to $0.5$ by convention.

Rankings need not be complete. Eq.~\eqref{eq:e_bar_rank} divides by $p - 1$ under the assumption that each expert's ranks span $[1, p]$. Under partial rankings, an expert who ranks~$m_i < p$ features has ranks spanning only $[1, m_i]$, leaving the uniform rescaling vulnerable to single-expert domination on thinly-covered features. We accommodate partial rankings via a coverage-weighted generalization of Eq.~\eqref{eq:e_bar_rank} that rescales each expert's ranks on their own range and shrinks the aggregated score toward $0.5$ in proportion to the missing-coverage fraction. The generalization reduces exactly to Eq.~\eqref{eq:e_bar_rank} under full rankings. Appendix~\ref{app:partial_rankings} gives the full derivation.

In all three formulations, the resulting vector $\bar{\boldsymbol{e}}$ serves as the parameter vector of the prior on the binary inclusion indicators $\bm{z}$. The MAP formulation in Section~\ref{sec:map} proceeds identically regardless of which elicitation format was used.

\subsection{Maximum A Posteriori Formulation}
\label{sec:map}

We apply the following prior and likelihood structure to the model parameters. Let~$\sigma_\varepsilon^2$ denote the noise variance, that is, the conditional variance of $y_i$ given $\bm{x}_i$. We specify
\begin{align*}
    y_i &\sim \mathcal{N}(\bm{x}_i^\top \boldsymbol{\beta},\; \sigma_\varepsilon^2), \\
    p(\boldsymbol{\beta}) &\propto 1, \\
    z_i &\sim \operatorname{Bernoulli}(\bar{e}_i), \qquad i = 1, \ldots, p,
\end{align*}
where all elements of $\bm{z}$ are independent, each with its own parameter $\bar{e}_i$. We choose a flat improper prior for $\boldsymbol{\beta}$ so that the likelihood contribution of $\boldsymbol{\beta}$ reduces to the standard normal log-likelihood. This makes the formulation as similar as possible to ordinary least-squares (OLS) estimation. Other choices of prior on $\boldsymbol{\beta}$, such as a Gaussian, would lead to additional regularization terms such as the $L_2$ penalty.

We begin with the joint posterior
\begin{align*}
    P(\boldsymbol{\beta}, \bm{z} \mid \bm{X}, \bm{y}) &= \frac{P(\bm{y} \mid \bm{X}, \boldsymbol{\beta})\, P(\boldsymbol{\beta})\, P(\bm{z})}{P(\bm{X}, \bm{y})}.
\end{align*}
Taking logarithms of the joint posterior, we have
\begin{align*}
    \log P(\boldsymbol{\beta}, \bm{z} \mid \bm{X}, \bm{y}) &= \log P(\bm{y} \mid \bm{X}, \boldsymbol{\beta}) + \log P(\boldsymbol{\beta}) + \log P(\bm{z}) - \log P(\bm{X}, \bm{y}).
\end{align*}
Since $\log P(\bm{X}, \bm{y})$ does not depend on $\boldsymbol{\beta}$ or $\bm{z}$, and the flat prior on $\boldsymbol{\beta}$ is a constant, both terms can be absorbed into a single constant $c$, yielding
\begin{align*}
    \log P(\boldsymbol{\beta}, \bm{z} \mid \bm{X}, \bm{y}) &= \log P(\bm{y} \mid \bm{X}, \boldsymbol{\beta}) + \log P(\bm{z}) + c.
\end{align*}
Expanding the Gaussian likelihood, we have
\begin{align*}
    P(\bm{y} \mid \bm{X}, \boldsymbol{\beta})
    &= \prod_{i=1}^n P(y_i \mid \bm{x}_i, \boldsymbol{\beta}) \\
    &= \prod_{i=1}^n \frac{1}{\sqrt{2\pi\sigma_\varepsilon^2}} \exp\!\left(-\frac{(y_i - \bm{x}_i^\top \boldsymbol{\beta})^2}{2\sigma_\varepsilon^2}\right).
\end{align*}
Taking the logarithm yields
\begin{align*}
    \log P(\bm{y} \mid \bm{X}, \boldsymbol{\beta})
    &= \sum_{i=1}^n \log\left( \frac{1}{\sqrt{2\pi\sigma_\varepsilon^2}} \exp\!\left(-\frac{(y_i - \bm{x}_i^\top \boldsymbol{\beta})^2}{2\sigma_\varepsilon^2}\right) \right) \\
    &= \sum_{i=1}^n \left( \log (2\pi\sigma_\varepsilon^2)^{-1/2} + \log \exp\!\left(-\frac{(y_i - \bm{x}_i^\top \boldsymbol{\beta})^2}{2\sigma_\varepsilon^2}\right) \right) \\
    &= \sum_{i=1}^n \left( -\tfrac{1}{2}\log(2\pi\sigma_\varepsilon^2) - \frac{(y_i - \bm{x}_i^\top \boldsymbol{\beta})^2}{2\sigma_\varepsilon^2} \right) \\
    &= -\frac{n}{2}\log(2\pi\sigma_\varepsilon^2) - \frac{1}{2\sigma_\varepsilon^2} \|\bm{y} - \bm{X}\boldsymbol{\beta}\|_2^2.
\end{align*}
Since $-\frac{n}{2}\log(2\pi\sigma_\varepsilon^2)$ depends on neither $\boldsymbol{\beta}$ nor $\bm{z}$, it is absorbed into $c$. Expanding the Bernoulli prior, we have
\begin{align*}
    \log P(\boldsymbol{\beta}, \bm{z} \mid \bm{X}, \bm{y})
    &= -\frac{1}{2\sigma_\varepsilon^2} \|\bm{y} - \bm{X}\boldsymbol{\beta}\|_2^2 + \log P(\bm{z}) + c\\
    &= -\frac{1}{2\sigma_\varepsilon^2} \|\bm{y} - \bm{X}\boldsymbol{\beta}\|_2^2 + \left( \log \prod_{i=1}^p \bar{e}_i^{z_i} (1-\bar{e}_i)^{1-z_i} \right) + c\\
    &= -\frac{1}{2\sigma_\varepsilon^2} \|\bm{y} - \bm{X}\boldsymbol{\beta}\|_2^2 + \left( \sum_{i=1}^p z_i \log \bar{e}_i + (1-z_i)\log(1-\bar{e}_i) \right) + c\\
    &= -\frac{1}{2\sigma_\varepsilon^2} \|\bm{y} - \bm{X}\boldsymbol{\beta}\|_2^2 + \left( \sum_{i=1}^p z_i \log \frac{\bar{e}_i}{1-\bar{e}_i} + \sum_{i=1}^p \log(1-\bar{e}_i) \right) + c.
\end{align*}
Because $\sum_{i=1}^p \log(1-\bar{e}_i)$ depends on neither $\boldsymbol{\beta}$ nor $\bm{z}$, it can be absorbed into the constant $c$, yielding
\begin{equation}
\label{eq:log_posterior}
    \log P(\boldsymbol{\beta}, \bm{z} \mid \bm{X}, \bm{y}) = -\frac{1}{2\sigma_\varepsilon^2} \|\bm{y} - \bm{X}\boldsymbol{\beta}\|_2^2 + \bm{z}^\top \log \frac{\bm{\bar{e}}}{\mathbf{1} - \bm{\bar{e}}} + c,
\end{equation}
where the division and logarithm are applied component-wise. Here $\sigma_\varepsilon^2$ is treated as a fixed constant estimated from the training data; how it is estimated is described in Section~\ref{sec:implementation}. To find the MAP estimator, we maximize Eq.~\eqref{eq:log_posterior}, or equivalently minimize its negation, which yields
\begin{equation}
\label{eq:opt_no_constraints}
    \boldsymbol{\beta}_{\text{MAP}},\, \bm{z}_{\text{MAP}} = \argmin_{\boldsymbol{\beta},\, \bm{z}} \left( \frac{1}{2\sigma_\varepsilon^2} \|\bm{y} - \bm{X}\boldsymbol{\beta}\|_2^2 - \bm{z}^\top \log \frac{\bm{\bar{e}}}{\mathbf{1} - \bm{\bar{e}}} \right).
\end{equation}

\paragraph{Boundary conditions.}
When $\bar{e}_a \to 0$ or $\bar{e}_a \to 1$, the objective function approaches $-\infty$ or $\infty$ respectively. While this theoretically yields the appropriate behavior, the infinities cause numerical issues in practice. To avoid this, boundary expert views $\bar{e}_a \in \{0, 1\}$ become hard MIP equality constraints. These constraints are: For all $i \in \{\,i : \bar{e}_i \in \{0, 1 \} \,\}$, we have $z_i = \bar{e}_i$.

\subsection{Final Optimization Problem}
\label{sec:opt_problem}

Using the constraint structure of Eq.~\eqref{eq:best_subsets_opt}, the EBBS optimization problem is
\begin{equation}
\label{eq:novel_opt}
\begin{aligned}
    \min_{\boldsymbol{\beta},\, \bm{z}}
    \quad & \frac{1}{2\sigma_\varepsilon^2} \|\bm{y} - \bm{X}\boldsymbol{\beta}\|_2^2 - \bm{z}^\top \log \frac{\bm{\bar{e}}}{\mathbf{1} - \bm{\bar{e}}} \\
    \textrm{s.t.}
    \quad & -\mathcal{M}_U z_i \leq \beta_i \leq \mathcal{M}_U z_i, \qquad i = 1, \ldots, p, \\
    & z_i \in \{0, 1\}, \qquad i = 1, \ldots, p, \\
    & \sum_{i=1}^p z_i \leq k, \\
    & z_i = \bar{e}_i, \qquad i \in \{\,i : \bar{e}_i \in \{0, 1 \} \,\}.
\end{aligned}
\end{equation}
This is a MIO problem and inherits the structure of the best-subsets MIO of \textcite{bertsimasBestSubsetSelection2016}, with the addition of a linear term $-\bm{z}^\top \log\!\left(\bm{\bar{e}} / (\mathbf{1} - \bm{\bar{e}})\right)$ in the objective. This term is computed from the expert views prior to solving and enters the objective as a fixed vector of coefficients on the binary selection variables. The formulation also includes the constraint $z_i = \bar{e}_i$ when $\bar{e}_i \in \{ 0, 1 \}$ to resolve numerical issues from boundary conditions. The MIP structure and constraint set are otherwise unchanged, so standard warm-starting and MIO algorithmic advances \parencite{bertsimasBestSubsetSelection2016, hazimehFastBestSubset2020} apply directly to EBBS. \textcite{bertsimasBestSubsetSelection2016} introduced alternative formulations that may perform better under certain regimes, such as the Specially Ordered Sets (SOS-1) formulation. These formulations can be extended similarly to \ref{eq:novel_opt} by adding $-\bm{z}^\top \log\!\left(\bm{\bar{e}} / (\mathbf{1} - \bm{\bar{e}})\right)$ to the respective objective functions. As in Eq.~\eqref{eq:best_subsets_opt}, the only tuning parameter requiring cross-validation is $k$.

\subsection{Implementation Details}
\label{sec:implementation}

\paragraph{Noise variance estimation.}
The MAP objective in Eq.~\eqref{eq:opt_no_constraints} requires the noise variance $\sigma_\varepsilon^2$ as a fixed input. Following \textcite{bertsimasBestSubsetSelection2016}, we standardize the columns of $\bm{X}$ to zero mean and unit $\ell_2$-norm. Unlike Bertsimas et al., who additionally normalize $\bm{y}$ to unit $\ell_2$-norm, we center $\bm{y}$ only. This keeps $\hat{\sigma}_\varepsilon^2$ interpretable as the actual noise variance in $\bm{y}$'s natural scale rather than a rescaled quantity that conflates noise variance with sample size. We estimate $\sigma_\varepsilon^2$ from the training data using OLS residuals as
\[
    \hat{\sigma}_\varepsilon^2 = \frac{1}{n-p-1} \|\bm{y} - \bm{X}\hat{\boldsymbol{\beta}}_{\text{OLS}}\|_2^2,
\]
where the $-1$ in $n-p-1$ accounts for the degree of freedom used by centering $\bm{y}$.

We treat $\hat{\sigma}_\varepsilon^2$ as a known constant, an empirical-Bayes plug-in, rather than placing an inverse-gamma prior and marginalizing it out. Under the conjugate inverse-gamma prior, marginalizing $\sigma_\varepsilon^2$ yields a multivariate Student-$t$ marginal likelihood in $\bm{y}$ \parencite{gelmanBayesianDataAnalysis2013}. The log of this term is not quadratic in $\boldsymbol{\beta}$. The Gaussian data term, by contrast, is $-\frac{1}{2\sigma_\varepsilon^2}\|\bm{y} - \bm{X}\boldsymbol{\beta}\|_2^2$, which is quadratic. For fixed $\bm{z}$, this quadraticity makes the inner minimization over $\boldsymbol{\beta}$ a convex quadratic program that MIO solvers handle through continuous relaxations. The non-quadratic Student-$t$ log term does not admit this structure, so a fully Bayesian treatment of $\sigma_\varepsilon^2$ is not compatible with the MIO formulation.

\paragraph{High-dimensional regime.}
When $n \leq p + 1$, the divisor $n - p - 1$ becomes non-positive and $\hat{\sigma}_\varepsilon^2$ is not well-defined. We fall back to an estimator that substitutes a sparse support for the full feature set. We fit a cross-validation-tuned Lasso on the training data, let $S$ denote its selected support, refit OLS on $\bm{X}_S$, and compute
\[
    \hat{\sigma}_\varepsilon^2 = \frac{1}{n - |S| - 1} \|\bm{y} - \bm{X}_S \hat{\boldsymbol{\beta}}_S\|_2^2.
\]
This preserves the denominator form of the original estimator, with $|S|$ in place of $p$, and is well-defined whenever $|S| + 1 < n$. This holds on the high-dimensional benchmarks considered in \textcite{bertsimasBestSubsetSelection2016}, such as leukemia ($n=72, p=1000$). In the rare case when $|S| + 1 \geq n$, or when the Lasso selects $|S| = 0$, we fall back further to the sample marginal variance $\operatorname{Var}(\bm{y})$. This estimates the marginal variance $\sigma_y^2$ rather than the noise variance $\sigma_\varepsilon^2$. The decomposition $\sigma_y^2 = \sigma_\varepsilon^2 + \operatorname{Var}(\bm{x}^\top \boldsymbol{\beta}^*)$ shows that $\sigma_y^2$ contains signal variance as well, so this fallback overestimates $\sigma_\varepsilon^2$. The overestimate shrinks $1/(2\hat{\sigma}_\varepsilon^2)$ and shifts relative influence toward the expert prior. The Lasso is used here only to estimate $\hat{\sigma}_\varepsilon^2$, not to perform feature selection.

\paragraph{Prior influence and sample size.}
Under any Gaussian likelihood, the residual sum of squares $\|\bm{y} - \bm{X}\boldsymbol{\beta}\|_2^2$ grows linearly with $n$, so the data term in Eq.~\eqref{eq:opt_no_constraints} contributes $O(n)$ information regardless of how $\bm{y}$ is normalized. The expert log-odds prior term is independent of $n$. As $n$ grows, the expert prior becomes progressively less influential on features where $\bar{e}_a \in (0, 1)$.

\section{Methodological Properties}
\label{sec:properties}

We now characterize the behavior of the expert prior term in the EBBS objective. The per-feature contribution of the prior to the objective is $-z_a \log\!\left(\bar{e}_a / (1 - \bar{e}_a)\right)$. Since $z_a \in \{0,1\}$, this term equals $-\log\!\left(\bar{e}_a / (1 - \bar{e}_a)\right)$ when $z_a = 1$ and $0$ when $z_a = 0$.

There are three cases depending on the value of $\bar{e}_a$:

\begin{proposition}
\label{prop:prior_behavior}
The expert prior term in Eq.~\eqref{eq:novel_opt} satisfies the following behavioral properties for each feature $a \in \{1, \ldots, p\}$:
\begin{enumerate}
    \item \emph{Expert support ($\bar{e}_a > 0.5$):} The log-odds term is positive, so $-z_a \log\!\left(\bar{e}_a / (1-\bar{e}_a)\right) < 0$ when $z_a = 1$, reducing the objective and incentivizing inclusion of feature $a$.
    \item \emph{Expert opposition ($\bar{e}_a < 0.5$):} The log-odds term is negative, so $-z_a \log\!\left(\bar{e}_a / (1-\bar{e}_a)\right) > 0$ when $z_a = 1$, increasing the objective and penalizing inclusion of feature $a$.
    \item \emph{Expert indifference ($\bar{e}_a = 0.5$):} The log-odds term is $0$, so the prior has no effect on the optimization for feature $a$.
\end{enumerate}
\end{proposition}

\begin{corollary}
\label{cor:reduction}
If $\bar{e}_a = 0.5$ for all $a \in \{1, \ldots, p\}$, the EBBS objective in Eq.~\eqref{eq:novel_opt} reduces to the best-subsets objective in Eq.~\eqref{eq:best_subsets_opt}, as the constant factor $1/\sigma_\varepsilon^{2}$ has no effect on the optimization when the linear expert term reduces to $0$ and the $z_i = \bar{e}_i$ constraints become inactive. This occurs when all expert estimates satisfy $e_{i,a} = 0.5$ for all $i$ and $a$, corresponding to experts holding no views on any feature.
\end{corollary}

\begin{proposition}
\label{prop:boundary_behavior}
EBBS incorporates expert views as intended: near-certain expert belief forces the corresponding selection decision, while moderate beliefs provide a soft incentive that interacts with the data.
\end{proposition}
\begin{proof}
The expert prior term exhibits the following boundary behavior:
\begin{enumerate}
    \item As $\bar{e}_a \to 1$, the log-odds $\log(\bar{e}_a / (1-\bar{e}_a)) \to +\infty$, driving the objective to $-\infty$ when $z_a = 1$, which forces $z_a^* = 1$ in any optimal solution.
    \item As $\bar{e}_a \to 0$, the log-odds $\log(\bar{e}_a / (1-\bar{e}_a)) \to -\infty$, driving the objective to $+\infty$ when $z_a = 1$, which forces $z_a^* = 0$ in any optimal solution.
\end{enumerate}
\end{proof}

Propositions~\ref{prop:prior_behavior} and~\ref{prop:boundary_behavior} together suggest that the expert prior provides a smooth interpolation between pure data-driven selection (at $\bar{e}_a = 0.5$) and deterministic enforcement of expert judgment at the boundaries $\bar{e}_a \in \{0, 1\}$. The strength of the expert influence on each feature is determined by $\bar{e}_a$, which is itself determined by the elicited probabilities through Eq.~\eqref{eq:e_bar}.

\begin{remark}
The objective in Eq.~\eqref{eq:novel_opt} can be interpreted as a penalized regression, where the penalty on feature $a$ is $-\log\!\left(\bar{e}_a / (1-\bar{e}_a)\right)$. Unlike the $L_1$ or $L_2$ penalties used in regularization methods, this penalty is applied to the binary inclusion indicator $z_a$ rather than directly to the coefficient $\beta_a$. It can be positive (discouraging inclusion) or negative (encouraging inclusion) depending on the expert consensus.
\end{remark}

\section{Empirical Results}
\label{sec:empirical}

\emph{Empirical results are forthcoming.} We outline the experimental design below.

\subsection{Synthetic Data}
\label{sec:synthetic}

We plan to validate EBBS on synthetic datasets generated following the framework of \textcite{bertsimasBestSubsetSelection2016}. Data are generated from $\bm{y} = \bm{X}\boldsymbol{\beta}^* + \boldsymbol{\varepsilon}$, where $\boldsymbol{\varepsilon} \sim \mathcal{N}(0, \sigma_\varepsilon^2 I)$, $\bm{X} \sim \mathcal{N}(0, I)$, and $\boldsymbol{\beta}^*$ is sparse with only $k^*$ nonzero coefficients. Because the true model is known, we can evaluate both out-of-sample mean squared error and feature recovery accuracy, measured by true positive rate, false positive rate, and $F_1$ score against the true support of $\boldsymbol{\beta}^*$. We vary the signal-to-noise ratio (low, medium, high), sample size $n$ and dimensionality $p$ including $p > n$ regimes, true sparsity level $k^*$, and the correlation structure of $\bm{X}$.

Since the true $\boldsymbol{\beta}^*$ is known in the synthetic setting, expert views are constructed as a linear mixture between the true signal and complete lack of expert views. For each expert $i$ and feature $a$, we set
\begin{equation}
\label{eq:expert_mixture}
    e_{i,a} = (1 - \lambda) \cdot \mathbbm{1}(\beta_a^* \neq 0) + \lambda \cdot 0.5,
\end{equation}
where $\lambda \in [0, 1]$ is a noise parameter. This is a convex combination, so $e_{i,a} \in [0,1]$ by construction for all values of $\lambda$. At $\lambda = 0$, experts are perfectly informed with $e_{i,a} \in \{0,1\}$, which matches the true support. At $\lambda = 1$, experts are completely uninformed with $e_{i,a} = 0.5$ for all features, and EBBS reduces to best subsets. At intermediate values, true features receive $e_{i,a} = 1 - \lambda/2$ and irrelevant features receive $e_{i,a} = \lambda/2$, both shrinking linearly toward $0.5$ as $\lambda$ increases. We evaluate the model on the grid $\lambda \in \{0, \, 0.2, \, 0.4, \, 0.6, \, 0.8, \,1.0\}$.

\paragraph{Expert disagreement.}
The construction in Eq.~\eqref{eq:expert_mixture} produces identical probability reports across all experts for each feature: every expert $i$ reports $e_{i,a} = 1 - \lambda/2$ for every true feature $a$ and $e_{i,a} = \lambda/2$ for every irrelevant feature. Under this construction the Poisson Binomial aggregation in Eq.~\eqref{eq:e_bar} collapses to a Binomial with a shared success probability. The number of experts $N_{\text{e}}$ then affects only the sharpness of the majority-vote threshold. The aggregation methodology introduced in Section~\ref{sec:agg_marginal} is therefore nontrivially utilized only when experts genuinely disagree. To validate EBBS under heterogeneous expert views, we augment Eq.~\eqref{eq:expert_mixture} with three disagreement models, each governed by a single additional parameter. All three strictly generalize the $\lambda$-mixture construction: at the zero setting of their disagreement parameter, each reduces exactly to Eq.~\eqref{eq:expert_mixture}.

\paragraph{Beta-distributed probability reports.}
The first model draws each expert's reported probability from a continuous distribution centered on the $\lambda$-mixture mean. We select the Beta distribution for two reasons. Its support is the unit interval, which matches the type space of probability reports without requiring truncation or rescaling. Its two shape parameters allow independent control over the first two moments, so the mean can be anchored at the $\lambda$-mixture value while the spread is tuned separately. Let $\mu_a = (1 - \lambda)\, \mathbbm{1}(\beta_a^* \neq 0) + \lambda \cdot 0.5$ denote the mean report at feature $a$ under Eq.~\eqref{eq:expert_mixture}. For expert $i$ we draw
\begin{equation}
\label{eq:expert_beta}
    e_{i,a} \sim \operatorname{Beta}\!\left(\mu_a\, c_a,\; (1 - \mu_a)\, c_a\right), \qquad c_a = \frac{\mu_a (1 - \mu_a)}{\eta^2} - 1,
\end{equation}
independently across experts and features, where $\eta \geq 0$ is a disagreement parameter measured on the standard-deviation scale. By construction, the resulting Beta distribution satisfies
\[
    \mathbb{E}[e_{i,a}] = \mu_a, \qquad \var(e_{i,a}) = \eta^2,
\]
whenever $\mu_a \in (0, 1)$ and $\eta < \sqrt{\mu_a(1-\mu_a)}$. The constant $c_a$ is the unique positive value that pins the variance to $\eta^2$ while preserving the mean at $\mu_a$. See Appendix~\ref{app:beta_moments} for the derivation. At $\eta = 0$, or at the $\lambda = 0$ boundary where $\mu_a \in \{0, 1\}$, the construction collapses to $e_{i,a} = \mu_a$ deterministically, recovering Eq.~\eqref{eq:expert_mixture}. As $\eta$ increases, reports become progressively more dispersed around $\mu_a$. This model captures experts whose reports share the correct mean but are individually noisy around it.

\paragraph{Bit-flip to no view.}
The second model captures experts who occasionally do not know the importance of a feature, but recognize their own uncertainty: rather than reporting a misleading value, the expert reports the uninformative value $0.5$. With probability $1 - p_{\text{flip}}$, the expert reports truthfully via the $\lambda$-mixture of Eq.~\eqref{eq:expert_mixture}, and with probability $p_{\text{flip}} \in [0, 0.5]$ they provide no view independently across experts and features. That is
\begin{equation}
\label{eq:expert_bitflip_noview}
    e_{i,a} = \begin{cases}
        (1 - \lambda)\, \mathbbm{1}(\beta_a^* \neq 0) + \lambda \cdot 0.5, & \text{with probability } 1 - p_{\text{flip}}, \\
        0.5, & \text{with probability } p_{\text{flip}}.
    \end{cases}
\end{equation} At $p_{\text{flip}} = 0$ every report is truthful and the construction reduces to Eq.~\eqref{eq:expert_mixture}. For nonzero $p_{\text{flip}}$, a no-view report carries no directional information about feature $a$'s membership in the true support: it neither supports nor opposes the truth. At~$\lambda = 0.2$, for instance, a feature actually in the true support receives probability $0.9$ from a truthful expert and $0.5$ from an expert with no view. This captures domain experts who admit ignorance on features outside their subspecialty rather than offering incorrect judgments.

\paragraph{Bit-flip belief errors.}
The third model is the adversarial counterpart of the no-view model: rather than providing no view, an erring expert confidently reports the \emph{opposite} of the truth. With probability $1 - p_{\text{flip}}$ the expert reports truthfully, and with probability $p_{\text{flip}} \in [0, 0.5]$ they report the $\lambda$-mixture applied to the complement of the true indicator again independently across experts and features. That is
\begin{equation}
\label{eq:expert_bitflip}
    e_{i,a} = \begin{cases}
        (1 - \lambda)\, \mathbbm{1}(\beta_a^* \neq 0) + \lambda \cdot 0.5, & \text{with probability } 1 - p_{\text{flip}}, \\
        (1 - \lambda)\, [\, 1 - \mathbbm{1}(\beta_a^* \neq 0) \,] + \lambda \cdot 0.5, & \text{with probability } p_{\text{flip}}.
    \end{cases}
\end{equation}
At $p_{\text{flip}} = 0$ every report is truthful and the construction reduces to Eq.~\eqref{eq:expert_mixture}. For nonzero $p_{\text{flip}}$, a flipped expert's report is the mirror image of the truthful report about $0.5$, actively contradicting the truth rather than merely failing to contribute. Continuing the example at $\lambda = 0.2$, a feature actually in the true support receives probability $0.1$ from a flipped expert, which is the reflection of the truthful $0.9$ and below the $0.5$ produced by an expert that provides no view. This captures domain experts who confidently provide incorrect beliefs.

\paragraph{Disagreement grids.}
For the Beta model we evaluate $\eta \in \{0, 0.05, 0.1, 0.2\}$, and for both bit-flip variants we evaluate $p_{\text{flip}} \in \{0, 0.1, 0.2, 0.3\}$. Each disagreement grid is crossed with the $\lambda$ grid, and at each grid point we measure the same out-of-sample prediction and feature-recovery metrics as under Eq.~\eqref{eq:expert_mixture}. Because having more than one expert is only nontrivial when experts have different views, the disagreement sweep also allows us to characterize the benefit of pooling multiple experts as a function of expert error rate.

We compare EBBS against stepwise selection, Lasso, Ridge Regression, Elastic Net, best subsets, and other regularization-based feature selection methods. We anticipate that EBBS may outperform benchmarks on feature selection accuracy at $\lambda = 0$, match best subsets at $\lambda = 1$, and degrade smoothly as $\lambda$ increases from $0$ to $1$. For the disagreement models, we anticipate that performance degrades smoothly as $\eta$ or $p_{\text{flip}}$ grows. The bit-flip belief error model is expected to be the most adverse of the three, since flipped experts actively contradict the truth. We also anticipate that increasing $N_{\text{e}}$ improves feature-recovery accuracy whenever the disagreement parameter is strictly positive.

\subsection{Real Data}
\label{sec:real_data}

For real datasets, where the true model is unknown, we approximate the ground truth in two steps. First, we select the cardinality~$k_{\text{pseudo}}$ by running best subsets on the entire dataset across a grid $k \in \{1, 2, \ldots, k_{\max}\}$, and choosing the $k$ that minimizes the squared error. Second, we refit best subsets on the full dataset at $k_{\text{pseudo}}$ and treat the recovered support as the true support. Expert views are then constructed using the same $\lambda$-mixture and disagreement models as in Section~\ref{sec:synthetic}, with $\mathbbm{1}(\beta_a^* \neq 0)$ replaced by $\mathbbm{1}(a \in \text{best subsets support at } k_{\text{pseudo}})$ throughout. The same $\lambda$ and disagreement grids are applied.

Within each train/validation/test split, the EBBS cardinality $k$ is chosen by cross-validation over the same grid $\{1, 2, \ldots, k_{\max}\}$. Here $k_{\max}$ is large enough to cover $k_{\text{pseudo}}$, so that the grid can reach every feature in the pseudo-true support and feature-recovery metrics are not structurally capped. We select the same two real datasets used in \textcite{bertsimasBestSubsetSelection2016}: the diabetes dataset\footnote{\url{https://www4.stat.ncsu.edu/~boos/var.select/diabetes.html}} and the leukemia gene expression dataset\footnote{\url{https://hastie.su.domains/CASI_files/DATA/leukemia.html}}.

\section{Conclusion}
\label{sec:conclusion}

We have introduced Expert-Implied Bayesian Best Subsets (EBBS), a method that extends the MIP formulation of best subset selection to incorporate domain-expert views on feature relevance through a MAP Bayesian framework. Expert assessments from multiple respondents are aggregated into a single per-feature probability using the Poisson Binomial distribution for marginal probability estimates, the pairwise win rate for pairwise comparisons, or the normalized mean rank for ordinal rankings. This probability enters the objective as a log-odds incentive term that incorporates the expert consensus into the selection decision. Methodological properties suggest that the method behaves as expected across the range of expert informativeness levels. It reduces to standard best subsets when experts are uninformative and enforces the expert-driven decision on any feature for which experts express full confidence.

Unlike existing expert-informed feature selection methods, which all operate within the Lasso family \parencite{zengIncorporatingPriorKnowledge2021, wuDomainKnowledgeenhancedVariable2022, destouchesFeatureSelectionPrior2023, zhangLLMLassoRobustFramework2025}, EBBS retains the global optimality guarantees of best subsets. This raises the possibility that, with well-specified expert priors, EBBS may recover the true support more reliably than heuristic or relaxation-based methods in high signal-to-noise settings. Empirical validation on synthetic and real datasets is forthcoming.

\paragraph{Acknowledgments.}
None.

\paragraph{Code availability.}
Implementations of all experiments will be released in a public git repository upon completion of the empirical study.

\paragraph{Data availability.}
All datasets used in this study are publicly available in the git repository.

\printbibliography

\appendix

\section{Derivation of the Beta Disagreement Model Moments}
\label{app:beta_moments}

This appendix derives the mean and variance of the Beta-distributed expert reports specified in Eq.~\eqref{eq:expert_beta} and shows that the value $c_a = \mu_a(1 - \mu_a)/\eta^2 - 1$ is the unique positive value that yields mean $\mu_a$ and variance $\eta^2$ simultaneously.

Fix a feature $a$ and let $X \sim \operatorname{Beta}(\alpha, \beta)$ with shape parameters $\alpha, \beta > 0$. The mean and variance of the Beta distribution are
\begin{equation}
\label{eq:beta_standard_moments}
    \mathbb{E}[X] = \frac{\alpha}{\alpha + \beta}, \qquad \var(X) = \frac{\alpha \beta}{(\alpha + \beta)^2 (\alpha + \beta + 1)}.
\end{equation}
Let $c_a = \alpha + \beta$. Then $\alpha = \mu_a c_a$ and $\beta = (1 - \mu_a) c_a$, with $c_a > 0$ because $\alpha, \beta > 0$ by definition. Under this substitution Eq.~\eqref{eq:beta_standard_moments} simplifies to
\begin{align*}
    \mathbb{E}[X] &= \frac{\mu_a c_a}{c_a} = \mu_a, \\
    \var(X) &= \frac{\mu_a c_a \cdot (1 - \mu_a) c_a}{c_a^2 (c_a + 1)} = \frac{\mu_a(1 - \mu_a)}{c_a + 1}.
\end{align*}
The mean is therefore $\mu_a$ for any $c_a > 0$. To force the variance to a target value $\eta^2$, we set
\[
    \frac{\mu_a(1 - \mu_a)}{c_a + 1} = \eta^2
\]
and solve for $c_a$, yielding
\begin{equation}
\label{eq:c_a_derivation}
    c_a = \frac{\mu_a(1 - \mu_a)}{\eta^2} - 1,
\end{equation}
which matches the choice in Eq.~\eqref{eq:expert_beta}. Because the map $c_a \mapsto \mu_a(1-\mu_a)/(c_a + 1)$ is strictly decreasing on $c_a > 0$, this is the unique positive value of $c_a$ that achieves the target variance.

A proper Beta distribution requires $c_a > 0$, which by Eq.~\eqref{eq:c_a_derivation} is equivalent to
\[
    \frac{\mu_a(1 - \mu_a)}{\eta^2} > 1, \qquad \text{or,} \qquad \eta < \sqrt{\mu_a(1 - \mu_a)}.
\]
This bound is maximal at $\mu_a = 0.5$, where it permits $\eta < 0.5$, and tightens to $\eta = 0$ at the boundaries $\mu_a \in \{0, 1\}$. At those boundaries the Beta degenerates and the construction reverts to the deterministic specification $e_{i,a} = \mu_a$, consistent with the~$\lambda = 0$ case of Eq.~\eqref{eq:expert_mixture}.

\section{Aggregation of Partial Rankings}
\label{app:partial_rankings}

This appendix derives the coverage-weighted generalization of Eq.~\eqref{eq:e_bar_rank} that accommodates partial rankings, in which expert $i$ may rank only a subset $S_i \subseteq \{1, \ldots, p\}$ of the features. Let $m_i = |S_i|$ denote the number of features expert $i$ ranks. The derivation requires two adjustments. First, we rescale each expert's ranks into $[0, 1]$ on their own rank range rather than the global $p$, removing the scale mismatch between experts who rank different numbers of features. Second, we shrink the aggregated score toward $0.5$ in proportion to the fraction of experts who ranked the feature, preventing an expert with thin coverage from dominating the aggregated score on a feature only they rated.

\paragraph{Per-expert normalized score.}
Let $m_i \geq 2$ and let $\pi_i(a)$ denote the rank that expert $i$ assigns to feature $a \in S_i$. The per-expert normalized score is
\begin{equation}
\label{eq:e_tilde_rank_partial}
    \tilde{e}_{i,a} = \frac{m_i - \pi_i(a)}{m_i - 1}, \qquad \tilde{e}_{i,a} \in [0, 1].
\end{equation}
A feature ranked first by expert $i$ receives $\tilde{e}_{i,a} = 1$, and a feature ranked last among the $m_i$ features they assessed receives $\tilde{e}_{i,a} = 0$. The per-expert score decreases linearly with rank between these extremes. Experts with $m_i = 1$ provide no comparative information and are excluded.

\paragraph{Aggregation across experts.}
The aggregated score for feature $a$ is the average over all $N_{\text{e}}$ experts. Experts not in $\mathcal{R}_a$ contribute the neutral value $0.5$ and experts in $\mathcal{R}_a$ contribute their normalized score $\tilde{e}_{i,a}$, which yields

\begin{equation}
\label{eq:e_bar_rank_partial}
    \bar{e}_a = \frac{1}{N_{\text{experts}}} \left( 0.5 \, (N_{\text{experts}} - |\mathcal{R}_a|) + \sum_{i \in \mathcal{R}_a} \tilde{e}_{i,a} \right).
\end{equation}

Since this is a convex combination of values in $[0, 1]$, the aggregated score satisfies $\bar{e}_a \in [0, 1]$ by construction. When all experts rank feature $a$, the second term vanishes and $\bar{e}_a$ is the unweighted average of per-expert scores. When no expert ranks feature $a$, the first sum is empty and $\bar{e}_a = 0.5$, matching the convention used in the full-rankings case. Eq.~\eqref{eq:e_bar_rank_partial} can equivalently be written as a coverage-weighted shrinkage toward $0.5$. Setting $w_a = |\mathcal{R}_a|/N_{\text{experts}}$, the formula is equivalent to
\[
    \bar{e}_a = w_a \cdot \bar{e}_a^{\text{ranked}} + (1 - w_a) \cdot 0.5, \qquad \bar{e}_a^{\text{ranked}} = \frac{1}{|\mathcal{R}_a|} \sum_{i \in \mathcal{R}_a} \tilde{e}_{i,a},
\]
which is well-defined whenever $|\mathcal{R}_a| \geq 1$ and exhibits the intended interpretation: features with thin coverage are pulled toward $0.5$ in proportion to the missing fraction, so a feature ranked first by a single expert out of $N_{\text{e}}$ contributes $1/N_{\text{experts}}$ of its way toward $1$ rather than reaching it outright.

\paragraph{Reduction to the full-rankings case.}
Eq.~\eqref{eq:e_bar_rank_partial} reduces to Eq.~\eqref{eq:e_bar_rank} when all experts provide full rankings of all $p$ features. In that case $m_i = p$ for every $i$ and $|\mathcal{R}_a| = N_{\text{experts}}$ for every $a$. The neutral contribution term vanishes, and the per-expert score becomes $\tilde{e}_{i,a} = (p - \pi_i(a)) / (p - 1)$. Averaging across experts and simplifying gives
\[
    \bar{e}_a = \frac{1}{N_{\text{experts}}} \sum_{i=1}^{N_{\text{experts}}} \frac{p - \pi_i(a)}{p - 1} = \frac{p - \bar{\pi}_a}{p - 1} = 1 - \frac{\bar{\pi}_a - 1}{p - 1},
\]
where $\bar{\pi}_a$ is the mean rank as defined in Section~\ref{sec:agg_ordinal}. The right-hand side is exactly Eq.~\eqref{eq:e_bar_rank}. The partial-rankings aggregation therefore strictly generalizes the full-rankings aggregation.

\end{document}